%% file: ICRA22_weize.tex
\documentclass[letterpaper, 10 pt, conference]{ieeeconf}
\IEEEoverridecommandlockouts
\usepackage{cite}
\usepackage{amsmath,amssymb,amsfonts}
\usepackage{algorithmic}
\usepackage{graphicx}
\usepackage{textcomp}
\usepackage{xcolor}

\newtheorem{theorem}{Theorem}[section]
\newtheorem{corollary}{Corollary}[theorem]
\newtheorem{lemma}[theorem]{Lemma}

\def\BibTeX{{\rm B\kern-.05em{\sc i\kern-.025em b}\kern-.08em
    T\kern-.1667em\lower.7ex\hbox{E}\kern-.125emX}}

\makeatletter
\DeclareRobustCommand{\textsupsub}[2]{{%
  \m@th\ensuremath{%
    ^{\mbox{\fontsize\sf@size\z@#1}}%
    _{\mbox{\fontsize\sf@size\z@#2}}%
  }%
}}
\makeatother

\title{{\LARGE \textbf{A Sufficient Condition for Convex Hull Property in General Convex Spatio-Temporal Corridors}}}

\author{Weize Zhang$^{1*}$, Peyman Yadmellat$^{1}$, and Zhiwei Gao$^{2}$
\thanks{$^{1}$Noah’s Ark Lab., Huawei Technologies Canada, Markham, Ontario,
Canada L3R 5A4.}
\thanks{$^{2}$Noah’s Ark Lab., Beijing Huawei Digital Technologies Co. Ltd.,
Beijing, China.}
\thanks{$^{*}$Correspondence: \texttt{\small {weize.zhang@huawei.com}}}
}

\begin{document}

\maketitle
\thispagestyle{empty}
\pagestyle{empty}

\input{sections/abstract.tex}
\input{sections/introduction.tex}
\input{sections/problem_formulation.tex}

\input{sections/proof.tex}
\input{sections/search_space_coverage.tex}
\input{sections/simulation_results.tex}
\input{sections/discussion.tex}
\input{sections/conclusions.tex}

\bibliographystyle{IEEEtran}
\bibliography{references}

\end{document}

%% file: sections/abstract.tex
\begin{abstract}
Motion planning is one of the key modules in autonomous driving systems to generate trajectories for self-driving vehicles to follow. A common motion planning approach is to generate trajectories within semantic safe corridors. The trajectories are generated by optimizing parametric curves (\textit{e.g.} Bezier curves) according to an objective function. To guarantee safety, the curves are required to satisfy the convex hull property, and be contained within the safety corridors. The convex hull property however does not necessary hold for time-dependent corridors, and depends on the shape of corridors. The existing approaches only support simple shape corridors, which is restrictive in real-world, complex scenarios. In this paper, we provide a sufficient condition for general convex, spatio-temporal corridors with theoretical proof of guaranteed convex hull property. The theorem allows for using more complicated shapes to generate spatio-temporal corridors and minimizing the uncovered search space to $O(\frac{1}{n^2})$ compared to $O(1)$ of trapezoidal corridors, which can improve the optimality of the solution. Simulation results show that using general convex corridors yields less harsh brakes, hence improving the overall smoothness of the resulting trajectories.
\end{abstract}

%% file: sections/introduction.tex
\section{Introduction}

Motion planning is one of the main components of an autonomous driving system with an essential role in generating feasible, safe trajectories. Semantic safe corridors and Bezier curves are commonly employed in optimization-based motion planning to guarantee safety, thanks to the convex hull property of Bezier curves. The idea is to generate a spatial safety corridor based on the surrounding occupancies, and use the corridor as a constraint in trajectory planning. The trajectories are then generated in $SLT$ (Frenet) frame~\cite{fan2018baidu} or $XYT$ (Cartesian) frame~\cite{sun2018fast}, aiming to optimize the overall swiftness, smoothness, and smartness \cite{fleury1995primitives}, 
under the constraints of continuity, vehicle dynamics limitation, traffic rules, and above all, safety \cite{lu2020adaptive}.  

One reason for using optimization-based approaches is that convex corridors
guarantee the convexity of the problem, allowing for solving the motion planning problem as a quadratic programming problem \cite{xu2014motion}, which are well-studied and can be implemented using off-the-shelf solvers~\cite{qian2016optimal}. Among different parametric curves, piece-wise Bezier curves are more popular as the convex hull property of Bezier curves\footnote{The Bezier curve defined by $n+1$ control points lies in the convex hull of the given control points, which is also called the control polygon.} \cite{aziz1990bezier} guarantees the curve segments to be confined within the corresponding 
safety corridor~\cite{gao2018online}. As a result, the safety of the generated trajectories are guaranteed. 

\begin{figure}[tbp]
\begin{center}
\includegraphics[width=8cm]{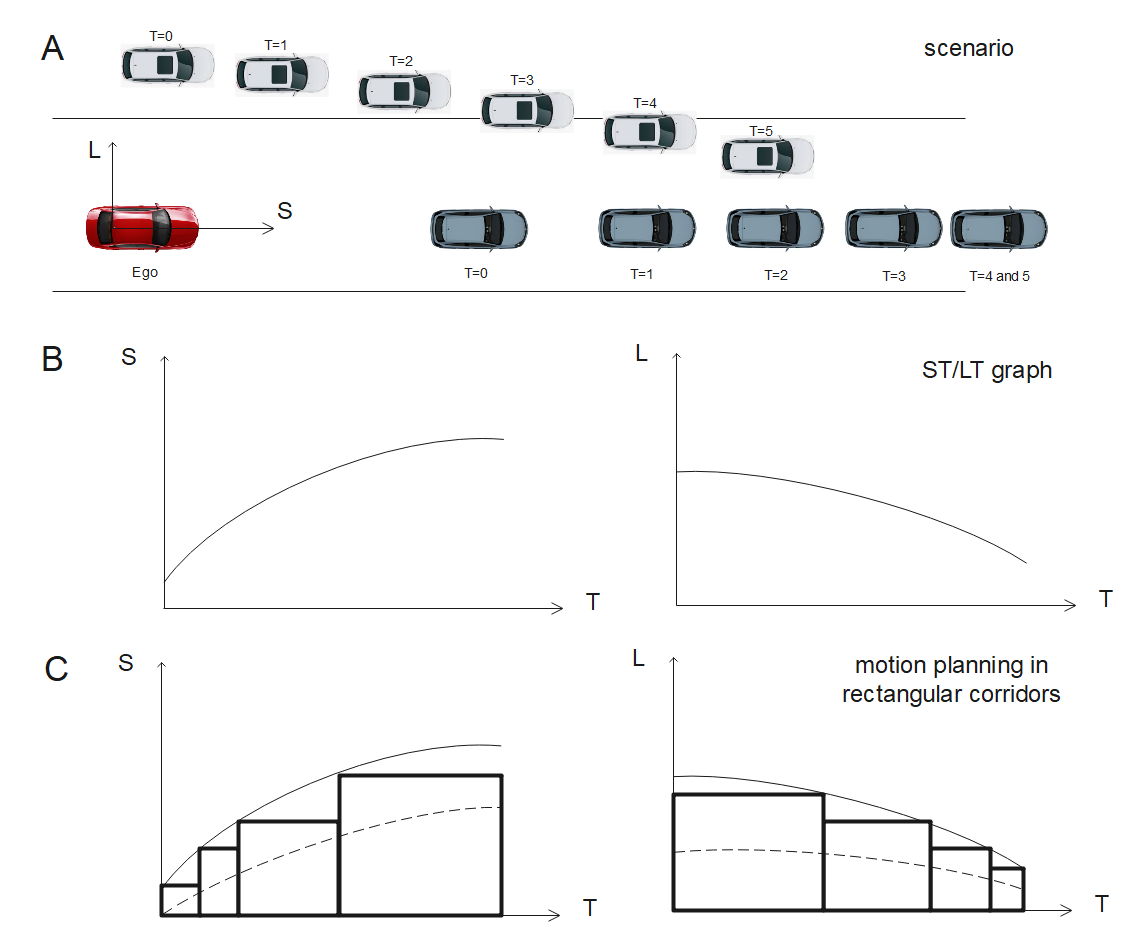}
\end{center}
\vspace{-0.2 in}
\caption{A) A scenario where the front vehicle decelerates, and the side vehicle pushes the ego vehicle to the right. B) The corresponding $ST/LT$ graphs. C) Motion planning in $ST/LT$ plane using rectangular spario-temporal corridors.}
\label{Fig:sampleProb}
\vspace{-0.2 in}
\end{figure}

The spatial corridors can be extended to spatio-temporal corridors to account for dynamic objects and predicted occupations in upcoming environments. When corridors consist of spatial dimension with no time-dependency (\textit{e.g.} defined in $XY$ plane or in $SL$ plane), the convex hull property holds, which allows for using general convex corridors for maximum search space coverage as utilized in~\cite{liu2018convex}.
However, when there is a time-dependency, \textit{i.e.} corridors are defined in spatio-temporal frames (\textit{e.g.} $ST$ or $LT$ planes), the convex hull property does not generally hold, as shown in Fig. \ref{Fig:ConvexHullNotHold}, imposing a limitation on the shape of corridors. This is stemmed from the fact that the spatial axis and the temporal axis are not equivalent. 

The convex hull property problem is often addressed by restricting the shape of corridors to be rectangular~\cite{ding2019safe}, where the convex hull property is ensured, resulting in a large uncovered search space and conservative or unsmoothed trajectories. The uncovered search space may lead to optimization failure in crowded scenarios, where no trajectories can be found within the corridors. As an extension, multiple rectangular corridors can be chained together to reduce the uncovered search space as illustrated in Fig.~\ref{Fig:sampleProb}, at the expense of significantly increasing the dimension of the optimization vector. This approach tends to be computationally intensive as it requires a sequential optimization steps to optimize a trajectory segment for each corridor and ensuring continuity from one segment to the next. Another method is to use a sampling based approach that supports general convex corridor shapes \cite{moghadam2020autonomous}. However the sampling-based approaches generally results in sub-optimal solutions, and there is no guarantee to find an optimal trajectory. A recent approach is to employ trapezoidal corridors \cite{li2021speed} with a sufficient condition for convex hull property to hold. The trapezoidal shape is effective in reducing the number of corridors, specifically in dynamic scenarios when there are objects with constant speed. While this is an improvement compared to the previous methods, the uncovered space can still be significant in certain scenarios. Fig.~\ref{Fig:comparison_of_shapes} A1 shows a case when there is a decelerating front vehicle, and Fig.~\ref{Fig:comparison_of_shapes} A2 shows a cut-in vehicle that with accelerating and decelerating speed profile. By comparing Fig.~\ref{Fig:comparison_of_shapes} B1/C1 and Fig.~\ref{Fig:comparison_of_shapes} B2/C2, one can see that in both cases, trapezoidal corridors have larger uncovered search space than that of general convex corridors, which may lead to unnecessary harsh brakes in motion planning.

\begin{figure}[tbp]
\begin{center}
\includegraphics[width=8cm]{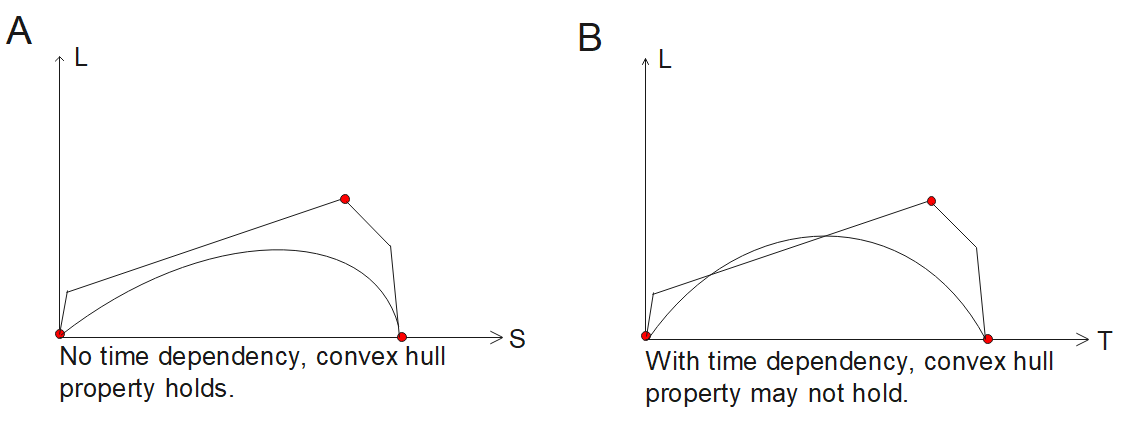}
\end{center}
\vspace{-0.2 in}
\caption{Convex hull property does not hold for spatio-temporal corridors of general convex shapes. A) The corridor is in $SL$ plane. Picking three control points along the boundary of the corridor, the generated
Bezier curve lies in the corridor. B) The corridor is in $LT$ plane. Picking three control points along the boundary of the corridor, the generated
Bezier curve is partially beyond the corridor. }
\label{Fig:ConvexHullNotHold}
\vspace{-0.2 in}
\end{figure}

The convex hull property in the time-dependent cases depends on the shape of the corridor, which precludes using arbitrary shaped corridors. Naturally, more complicated shapes cover more search space and yield better optimization results. An ideal case is to use general convex-shaped corridors. As the main contribution of this paper, we prove a sufficient condition to guarantee the convex hull property for general convex corridors in spatio-temporal frames. The use of general convex corridors can shrink the uncovered search space to $O(\frac{1}{n^2})$, compared to $O(1)$ of trapezoidal corridors.

This paper is organized as follows: Section \ref{ProblemFormulation} converts the main theorem into an equivalent theorem. Section \ref{Proof} decomposes the theorem into several lemmas and proves them by recurrence. Section \ref{SearchSpaceCoverage} shows that the uncovered search space is $O(\frac{1}{n^2})$ under continuity assumptions. Section \ref{Simulation} compares the proposed approach with the state of the art approach quantitatively under simulated scenarios. Section \ref{Discussion} provides deeper analysis to the proposed approach. Section \ref{CONCLUSIONS} presents our conclusions.

\begin{figure}[tbp]
\begin{center}
\includegraphics[width=8.5cm]{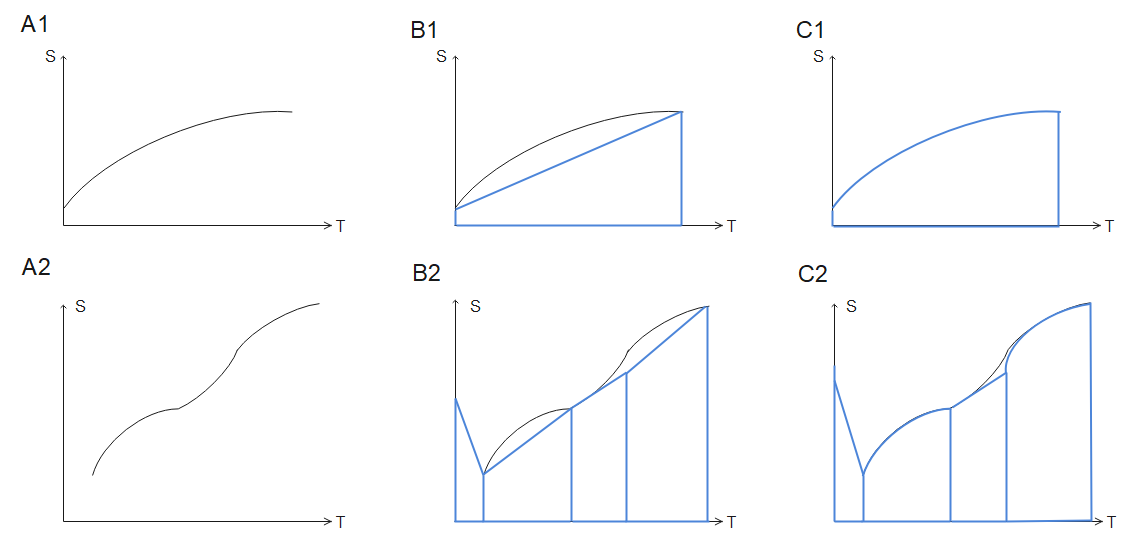}
\end{center}
\vspace{-0.2 in}
\caption{The advantage of general convex corridors is that they cover more search space. A1) $ST$ graph of a decelerating front vehicle. A2) $ST$ graph of a cut-in vehicle that switches between accelerating and decelerating. B1) The coverage of 1 trapezoidal corridor. B2) the coverage of 4 trapezoidal corridors. C1) The coverage of 1 general convex corridor. C2) The coverage of 4 general convex corridors.}
\label{Fig:comparison_of_shapes}
\vspace{-0.2 in}
\end{figure}

%% file: sections/problem_formulation.tex
\section{Problem Formulation} \label{ProblemFormulation}

\begin{figure}[tbp]
\begin{center}
\includegraphics[width=8cm]{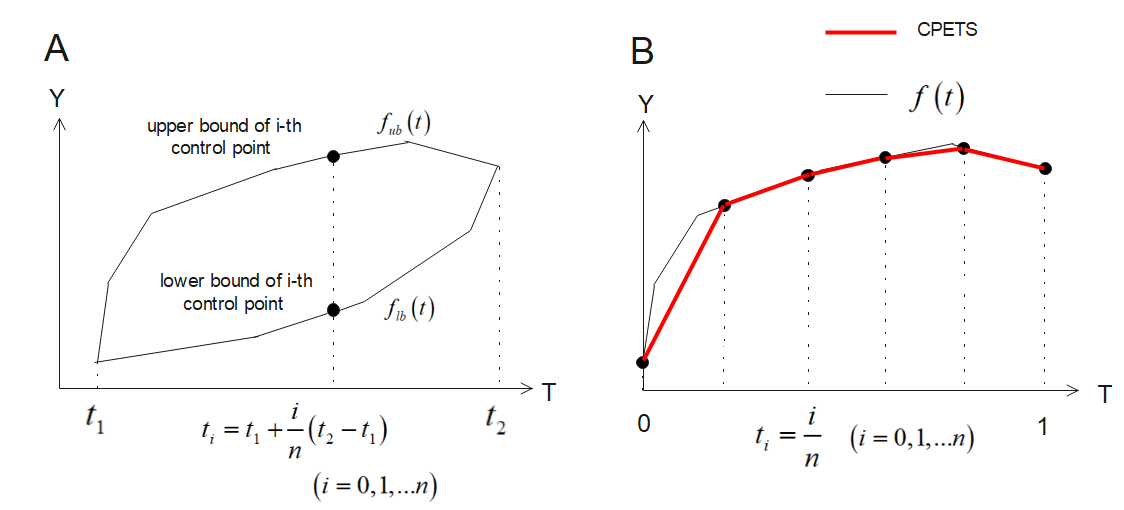}
\end{center}
\vspace{-0.2 in}
\caption{A) The range of the i-th control point that guarantees convex hull property. B) The problem scaled back to [0, 1]. It suffices to consider the concave upper bound function $f(t)$ and the control points picked along the upper bound with equal time spacing. The poly-line of the control polygon formed with these control points is defined as CPETS (in red poly-line).}
\label{Fig:probFormulation}
\vspace{-0.2 in}
\end{figure}

Considering a convex spatio-temporal corridor defined in $\left[t_{1}, t_{2}\right]$. The upper bound function $f_{ub}(t)$ is concave, while the lower bound function $f_{lb}(t)$ is convex. We pick $n+1$ control points in the corridor to form a scaled Bezier curve ~\cite{ding2019safe} of degree $n$ which is defined in $\left[t_{1}, t_{2}\right]$. The vertical axis label $Y$ is a generic spacial label, which could represent $X$, $Y$, $S$, $L$, or any spacial variable, as shown in Fig. \ref{Fig:probFormulation} A. The vertical coordinate of the i-th control point is noted as $s_{i}$. We want to find a sufficient condition to make sure the Bezier curve lies in the corridor.

\begin{theorem}
If the series $s_{i}$ (i=0,\;1,\;...,\;n) satisfies
\begin{equation}\begin{split}
s_{i} \in \left[f_{lb}(t_{1}+\frac{i}{n}(t_{2}-t_{1})), f_{ub}(t_{1}+\frac{i}{n}(t_{2}-t_{1}))\right]
\label{eq.0}
\end{split}\end{equation}
the scaled Bezier curve generated from control points lies in the corridor.
\label{theorem:baseTheorem}
\end{theorem}

Theorem \ref{theorem:baseTheorem} shows that it suffices to pick the control points with equal time spacing, and the range of each control point is the value of the upper bound/lower bound functions at the corresponding time. 

To simplify the proof of theorem \ref{theorem:baseTheorem}, without loss of generality, we can scale the time interval back to $\left[0, 1\right]$, consider the upper bound function $f(t)$ only, and pick the control points along the upper bound function with equal time spacing (which will generate the highest possible Bezier curve). The vertical coordinate of the i-th control point thus satisfies $s_{i}=f(\frac{i}{n})$. The poly-line of the control polygon formed with these control points is defined as CPETS (control polygon of equal time spacing), as shown in the red poly-line in Fig. \ref{Fig:probFormulation} B. Due to the property of concave functions, we have $\text{CPETS} \leq f(t)$. This yields theorem \ref{theorem:equivalentTheorem}, which is equivalent to theorem \ref{theorem:baseTheorem}.

\begin{theorem}
$\forall t\in\left[0, 1\right]$, the Bezier curve $C(t)=\sum_{i=0}^{n}s_{i}B\textsupsub{$n$}{$i$}(t)$ satisfies $C(t)$$\leq$CPETS$\leq$$f(t)$, where $f(t)$ is a concave function, $s_{i}=f(\frac{i}{n})$, and 
$B\textsupsub{$n$}{$i$}(t)$ is the Bernstein polynomial which is defined as $B\textsupsub{$n$}{$i$}(t)=C\textsupsub{$i$}{$n$}t^{i}(1-t)^{n-i}$.
\label{theorem:equivalentTheorem}
\end{theorem}

We will decompose the proof of theorem \ref{theorem:equivalentTheorem} into several lemmas. In this paper, all terms share the same definition as they are defined in theorem~\ref{theorem:equivalentTheorem}, unless stated otherwise.

%% file: sections/proof.tex
\section{Proof of the Main Theorem} \label{Proof}

\subsection{Lemmas}

\begin{lemma}
For two $C\textsupsub{$1$}{}$ functions $f_{1}(t)$ and $f_{2}(t)$, if  $f_{1}(t_{0}) \geq f_{2}(t_{0})$ and $\forall t \geq t_{0}$ (resp. $\forall t \leq t_{0}$), 
$f_{1}^{\prime}(t) \geq f_{2}^{\prime}(t)$ (resp. $f_{1}^{\prime}(t) \leq f_{2}^{\prime}(t)$), we have $f_{1}(t) \geq f_{2}(t)$ $\forall t \geq t_{0}$ (resp. $\forall t \leq t_{0}$).
\label{lemma:simpleAscending}
\end{lemma}

\begin{proof}
Using Lagrange’s mean value theorem, $\forall t > t_{0}$,
$\exists \xi > t_{0}$ such that
\begin{equation}\begin{split}
&f_{1}(t)-f_{2}(t) \\
&=f_{1}(t_{0})-f_{2}(t_{0}) + (f_{1}^{\prime}(\xi)-f_{2}^{\prime}(\xi))(t-t_{0}) \ge 0
\label{eq.1}
\end{split}\end{equation}%
\end{proof}

\begin{lemma}
Suppose $s_{i}$ is an ascending series, then the Bezier curve $C(t)=\sum_{i=0}^{n}s_{i}B\textsupsub{$n$}{$i$}(t)$ is an ascending function.
\label{lemma:bezierAscending}
\end{lemma}

\begin{proof}
Taking the derivative of $C(t)$,
\begin{equation}\begin{split}
C^{\prime}(t)=n\sum_{i=0}^{n-1}(s_{i+1}-s_{i})B\textsupsub{$n-1$}{$i$}(t) \ge 0
\end{split}\end{equation}
\end{proof}

\begin{lemma}
Suppose $s_{i}$ is a concave (convex) series, i.e. $s_{i+2}-2s_{i+1}+s_{i} \le 0$ (resp. $\ge 0$) $\forall i \le n-2$,  then the Bezier curve $C(t)=\sum_{i=0}^{n}s_{i}B\textsupsub{$n$}{$i$}(t)$ is a concave (convex) function.
\label{lemma:bezierConcave}
\end{lemma}

\begin{proof}
Taking the second derivative of $C(t)$,
\begin{equation}\begin{split}
C^{(2)}(t)=n(n-1)\sum_{i=0}^{n-2}(s_{i+2}-2s_{i+1}+s_{i})B\textsupsub{$n-2$}{$i$}(t) \le 0
\end{split}\end{equation}
\end{proof}

\begin{corollary}
Suppose $s_{i}$ is a concave (convex) series, the Bezier curve $C(t)=\sum_{i=0}^{n}s_{i}B\textsupsub{$n$}{$i$}(t)$ passes the first and the last control point, and is inferior than the first and the last segment of the CPETS.
\label{corollary:tangentFistLastPoint}
\end{corollary}

\begin{proof}
The proof is trivial by applying lemma \ref{lemma:simpleAscending} and lemma \ref{lemma:bezierConcave}.
\end{proof}

\begin{lemma}
$\forall n > 1$, if the upper bound function $f(t)$ defined in $\left[0,1\right]$ is ascending (resp. descending) and concave, for series $s_{i}=f(\frac{i}{n})$, $\forall j \in \left[0, n-1\right]$ and 
$\forall t \in \left[\frac{j}{n}, \frac{j+1}{n}\right]$, we have $C(t)=\sum_{i=0}^{n}s_{i}B\textsupsub{$n$}{$i$}(t)$ $\le$ $s_{j}+n(s_{j+1}-s{j})(t-\frac{j}{n})$, i.e. the Bezier curve is bounded by the CPETS.
\label{lemma:mainLemma}
\end{lemma}

\begin{proof}
We only need to prove lemma \ref{lemma:mainLemma} for the case that $f(t)$ is ascending. For the descending case, it suffices to replace $t$ with $(1-t)$ and it returns to the ascending case, thanks to a property of Bezier curve that its orientation can be reversed and its control points keep unchanged by replacing  $t$ with $(1-t)$.

The series $s_{i}$ is ascending and concave due to the properties of $f(t)$. Applying lemma \ref{lemma:bezierAscending} and lemma \ref{lemma:bezierConcave}, $C(t)$ is ascending and concave.
We will prove lemma \ref{lemma:mainLemma} by recurrence. For $n \le 2$, it is easy to verify that lemma \ref{lemma:mainLemma} holds using lemma \ref{lemma:simpleAscending}.

Suppose the lemma holds for $n=k$ $(k > 2)$, when $n=k+1$, using the recurrence definition of Bezier curve:

\begin{equation}\begin{split}
C(t)&=\sum_{i=0}^{k+1}s_{i}B\textsupsub{$k+1$}{$i$}(t) \\
&=(1-t)\sum_{i=0}^{k}s_{i}B\textsupsub{$k$}{$i$}(t)+t\sum_{i=0}^{k}s_{i+1}B\textsupsub{$k$}{$i$}(t)
\label{eq:BezierRecurDefinition}
\end{split}\end{equation}

We only need to show that lemma \ref{lemma:mainLemma} holds for intervals $\left[\frac{1}{k+1}, \frac{2}{k+1}\right]$, $\left[\frac{2}{k+1}, \frac{3}{k+1}\right]$, ... , 
$\left[\frac{k-1}{k+1}, \frac{k}{k+1}\right]$, because for the first and the last interval, lemma \ref{lemma:mainLemma} holds due to corollary \ref{corollary:tangentFistLastPoint}.

The recurrence definition (\ref{eq:BezierRecurDefinition}) means that $C(t)$ is a combination of two lower degree Bezier curves. One is formed with $s_{0}, s_{1},$ $...$ $, s_{k}$, the other is formed with $s_{1}, s_{2},$ $...$ $, s_{k+1}$. With the assumption of recurrence, $\forall j \in \left[1, k-1\right]$ and 
$\forall t \in \left[\frac{j}{k}, \frac{j+1}{k}\right]$, we have

\begin{equation}\begin{split}
C(t) \le D(t)
\label{eq:BezierRecurIniquality}
\end{split}\end{equation}
where
\begin{equation}\begin{split}
D(t) = (1-t)(s_{j}+k(s_{j+1}-s_{j})(t-\frac{j}{k})) \\
+t(s_{j+1}+k(s_{j+2}-s_{j+1})(t-\frac{j}{k}))
\label{eq:BezierRecurIniqualityRightSide}
\end{split}\end{equation}

One can easily verify that $D(\frac{j}{k}) = s_{j}$, $D^{\prime}(\frac{j}{k}) = (k+1)(s_{j+1}-s_{j})$, and $D(t)$ is concave. According to lemma \ref{lemma:simpleAscending} and (\ref{eq:BezierRecurIniquality}), we have
\begin{equation}\begin{split}
C(t) &\le D(t) \le s_{j}+(k+1)(s_{j+1}-s_{j})(t-\frac{j}{k}) \\
& \forall t \in \left[\frac{j}{k}, \frac{j+1}{k}\right]
\label{eq:unshiftedInequality}
\end{split}\end{equation}

Now we will use the monotony of $C(t)$. $\forall t \in \left[\frac{j}{k+1}, \frac{j+1}{k+1}\right]$, note $\Delta t = \frac{j}{k} - \frac{j}{k+1} > 0$, we have $t+\Delta t \in \left[\frac{j}{k}, \frac{j+1}{k}\right]$. As we know $C(t)$ is ascending, from (\ref{eq:unshiftedInequality}) we have

\begin{equation}\begin{split}
C(t) &\le C(t+\Delta t) \\
&\le s_{j}+(k+1)(s_{j+1}-s_{j})(t + \Delta t - \frac{j}{k}) \\
&= s_{j}+(k+1)(s_{j+1}-s_{j})(t - (\frac{j}{k} - \frac{j}{k} + \frac{j}{k+1})) \\
&= s_{j}+(k+1)(s_{j+1}-s_{j})(t - \frac{j}{k+1})
\label{eq:shiftedInequality}
\end{split}\end{equation}

This concludes the proof for $n=k+1$. Lemma \ref{lemma:mainLemma} is thus proven.
\end{proof}

\begin{corollary}
$\forall n > 1$, if the upper bound function $f(t)$ defined in $\left[0,1\right]$ is ascending (resp. descending) and concave, for series $s_{i}=f(\frac{i}{n})$, we have $C(t)=\sum_{i=0}^{n}s_{i}B\textsupsub{$n$}{$i$}(t)$ $\le$ $f(t)$.
\label{corollary:monotoneLowerThanFt}
\end{corollary}

\begin{proof}
Using lemma \ref{lemma:mainLemma}, the proof is trivial due to the fact that CPETS is bounded by $f(t)$.
\end{proof}

\begin{lemma}
$\forall n > 1$, if the concave upper bound function $f(t)$ defined in $\left[0,1\right]$ is ascending in $\left[0, t_{0}\right]$ and is descending in $\left[t_{0}, 1\right]$, for series $s_{i}=f(\frac{i}{n})$, we have $C(t)=\sum_{i=0}^{n}s_{i}B\textsupsub{$n$}{$i$}(t)$ $\le$ $f(t)$.
\label{lemma:ascendingNdescending}
\end{lemma}

\begin{proof}
Define two functions $f_{1}(t)$ and $f_{2}(t)$ as shown in (\ref{eq:newFunction1}) and (\ref{eq:newFunction2}):
\begin{equation}\begin{split}
f_{1}(t) = \left\{
\begin{array}{l}
f(t) \; if \; t \in \left[0, t_{0}\right] \\
f(t_{0}) \; \text{otherwise}
\end{array}
\right.  
\label{eq:newFunction1}
\end{split}\end{equation}

\begin{equation}\begin{split}
f_{2}(t) = \left\{
\begin{array}{l}
f(t) \; if \; t \in \left[t_{0}, 1\right] \\
f(t_{0}) \; \text{otherwise}
\end{array}
\right.  
\label{eq:newFunction2}
\end{split}\end{equation}

It is easy to verify that $f_{1}(t)$ is ascending, $f_{1}(t) \ge f(t)$, $f_{2}(t)$ is descending, $f_{2}(t) \ge f(t)$, and the point-wise minimum function of $f_{1}(t)$ and $f_{2}(t)$ is $f(t)$, as shown in (\ref{eq:pointWiseMiminum}).

\begin{figure}[tbp]
\begin{center}
\includegraphics[width=6cm]{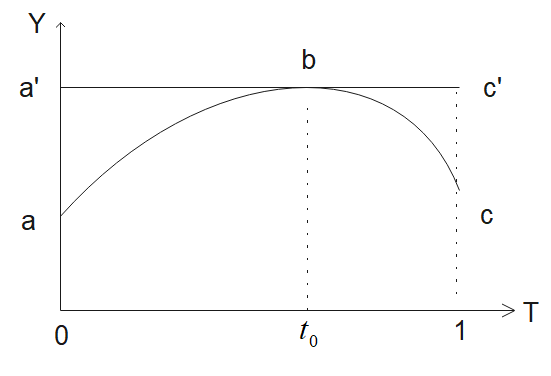}
\end{center}
\vspace{-0.2 in}
\caption{Graphs of functions $f_{1}(t)$ (curve $abc'$), $f_{2}(t)$ (curve $a'bc$) and $f(t)$ (curve $abc$).}
\label{Fig:ascendingNdecending}
\vspace{-0.2 in}
\end{figure}

\begin{equation}\begin{split}
f(t) = \text{min} \left\{f_{1}(t), f_{2}(t)\right\}
\label{eq:pointWiseMiminum}
\end{split}\end{equation}

To better illustrate the idea of this proof, Fig. \ref{Fig:ascendingNdecending} shows the graphs of functions $f_{1}(t)$, $f_{2}(t)$ and $f(t)$. The curve $abc$ represents $f(t)$, $abc'$ represents $f_{1}(t)$, and $a'bc$ represents $f_{2}(t)$. Our objective is to bound $C(t)$ with $f_{1}(t)$ and $f_{2}(t)$ so that we can take the common part of the bounds, which is $f(t)$.

Applying corollary \ref{corollary:monotoneLowerThanFt}, we have

\begin{equation}\begin{split}
\left\{
\begin{array}{l}
\sum_{i=0}^{n}f_{1}(\frac{i}{n})B\textsupsub{$n$}{$i$}(t) \le f_{1}(t) \\
\sum_{i=0}^{n}f_{2}(\frac{i}{n})B\textsupsub{$n$}{$i$}(t) \le f_{2}(t)
\end{array}
\right. 
\label{eq:lowerThanNewDefinedFunc}
\end{split}\end{equation}

As $f_{1}(t) \ge f(t)$ and $f_{2}(t) \ge f(t)$, we have

\begin{equation}\begin{split}
\left\{
\begin{array}{l}
\sum_{i=0}^{n}f(\frac{i}{n})B\textsupsub{$n$}{$i$}(t) \le f_{1}(t) \\
\sum_{i=0}^{n}f(\frac{i}{n})B\textsupsub{$n$}{$i$}(t) \le f_{2}(t)
\end{array}
\right. 
\label{eq:lowerThanNewDefinedFunc2}
\end{split}\end{equation}

From (\ref{eq:pointWiseMiminum}) and (\ref{eq:lowerThanNewDefinedFunc2}), we have

\begin{equation}\begin{split}
\sum_{i=0}^{n}f(\frac{i}{n})B\textsupsub{$n$}{$i$}(t) \le \text{min} \left\{f_{1}(t), f_{2}(t) \right\} = f(t)
\label{eq:lowerThanNewDefinedFunc3}
\end{split}\end{equation}

This concludes the proof.
\end{proof}

\begin{corollary}
$\forall n > 1$, if the concave upper bound function $f(t)$ defined in $\left[0,1\right]$ is ascending in $\left[0, t_{0}\right]$ and is descending in $\left[t_{0}, 1\right]$, for series $s_{i}=f(\frac{i}{n})$, $\forall j \in \left[0, n-1\right]$ and 
$\forall t \in \left[\frac{j}{n}, \frac{j+1}{n}\right]$, we have $C(t)=\sum_{i=0}^{n}s_{i}B\textsupsub{$n$}{$i$}(t)$ $\le$ $s_{j}+n(s_{j+1}-s{j})(t-\frac{j}{n})$, i.e. the Bezier curve is bounded by the CPETS.
\label{lemma:ascendingNdescendingCPETS}
\end{corollary}

\begin{proof}
One can notice that the CPETS itself can be treated as a concave upper bound function, and the CPETS of a CPETS is still itself. The proof is trivial by applying lemma \ref{lemma:ascendingNdescending}.
\end{proof}

\subsection{Proof of Theorem \ref{theorem:equivalentTheorem}}

\begin{proof}
For the concave upper bound function $f(t)$, there is at most three possibilities: $f(t)$ is ascending, or descending, or ascending then descending. Corollary \ref{lemma:ascendingNdescendingCPETS} and lemma \ref{lemma:mainLemma} cover all possibilities. This concludes the proof.
\end{proof}

%% file: sections/search_space_coverage.tex
\section{Search Space Coverage} \label{SearchSpaceCoverage}
We have proven the safety of the proposed method of choosing control points. Now we need to analyze the advantage of the proposed method, i.e. more search space coverage. From theorem \ref{theorem:baseTheorem} we known that $\sum_{i=0}^{n}f(\frac{i}{n})B\textsupsub{$n$}{$i$}(t)$ ($f(t)$ is the concave upper bound function) is the highest possible Bezier curve. It suffices to analyze the difference between the highest possible Bezier curve and $f(t)$. We are going to prove that the difference is $O(\frac{1}{n^{2}})$ if $f(t)$ is twice differentiable.

\begin{lemma}
For two twice differentiable functions $f_{1}(t)$ and $f_{2}(t)$ defined in $\left[t_{1}, t_{2}\right]$, if $\exists t_{0} \in \left[t_{1}, t_{2}\right]$ such that $f_{1}(t_{0}) = f_{2}(t_{0})$ and $f^{\prime}_{1}(t_{0}) = f^{\prime}_{2}(t_{0})$ , we have $f_{1}(t) - f_{2}(t)$ = $O((t_{2}-t_{1})^2)$.
\label{lemma:simpleTangentO2Diff}
\end{lemma}

\begin{proof}
The proof is trivial by applying Taylor's formula on $f_{2}(t) - f_{1}(t)$.
\end{proof}

\begin{lemma}
If $g_{1}(t)$ is a twice differentiable function defined in $\left[t_{1}, t_{2}\right]$, and $g_{2}(t)$ is the linear function passing $(t_{1}, g_{1}(t_{1}))$ and $(t_{2}, g_{1}(t_{2}))$, then $g_{1}(t) - g_{2}(t)$ $=$ $O((t_{1} - t_{2})^{2})$  
\label{lemma:2ndOrderDiff}
\end{lemma}

\begin{proof}
Note $h(t) = g_{1}(t) - g_{2}(t)$, we have obviously $h(t_{1}) = h(t_{2})=0$ and $h(t)$ is twice differentiable. $\forall t \in \left[t_{1}, t_{2}\right]$, from Taylor's formula we have
{
\setlength\abovedisplayskip{1pt}
\setlength\belowdisplayskip{1pt}
\begin{equation}\begin{split}
0=h(t_{1}) = h(t) + h^{(1)}(t)(t_{1} - t) + O((t_{2}-t_{1})^{2})
\label{eq:Taylor1}
\end{split}\end{equation}
}
{
\setlength\abovedisplayskip{1pt}
\setlength\belowdisplayskip{1pt}
\begin{equation}\begin{split}
0=h(t_{2}) = h(t) + h^{(1)}(t)(t_{2} - t) + O((t_{2}-t_{1})^{2})
\label{eq:Taylor2}
\end{split}\end{equation}
}
Calculating the difference of (\ref{eq:Taylor1}) and (\ref{eq:Taylor2}) yields
{
\setlength\abovedisplayskip{1pt}
\setlength\belowdisplayskip{1pt}
\begin{equation}\begin{split}
0=h^{(1)}(t)(t_{1}-t_{2})+O((t_{2}-t_{1})^{2})
\label{eq:Taylor3}
\end{split}\end{equation}
}
From (\ref{eq:Taylor3}) we know that $h^{(1)}(t)(t_{1}-t_{2})$ = $O((t_{2}-t_{1})^{2})$, so $h^{(1)}(t)(t_{1} - t)$ = $O((t_{2}-t_{1})^{2})$ $\forall t \in \left[t_{1}, t_{2}\right]$. From (\ref{eq:Taylor1}) we know that $h(t)$ = $-h^{(1)}(t)(t_{1} - t)$ + $O((t_{2}-t_{1})^{2})$ = $O((t_{2}-t_{1})^{2})$. This concludes the proof.
\end{proof}

\begin{corollary}
Suppose $s_{i}$ is a concave (convex) series, the Bezier curve $C(t)=\sum_{i=0}^{n}s_{i}B\textsupsub{$n$}{$i$}(t)$ is inferior than the first and the last segment of the CPETS, and the difference is $O(\frac{1}{n^{2}})$.
\label{corollary:tangentFistLastPointDiff}
\end{corollary}

\begin{proof}
The proof is trivial by applying corollary \ref{corollary:tangentFistLastPoint} and lemma \ref{lemma:simpleTangentO2Diff}.
\end{proof}

\begin{corollary}
If the concave upper bound function $f(t)$ is twice differentiable, the difference between $f(t)$ and the CPETS is $O(\frac{1}{n^{2}})$, where $n$ is the number of control points.
\label{corollary:diffUpperBoundCPETS}
\end{corollary}

\begin{proof}
The proof is obvious by applying lemma \ref{lemma:2ndOrderDiff} on interval $\left[\frac{i}{n}, \frac{i+1}{n}\right]$ ($i=0,\;1,\;...\;,\;n-1$).
\end{proof}

With corollary \ref{corollary:diffUpperBoundCPETS}, it suffices to prove that the difference between the highest possible Bezier curve $\sum_{i=0}^{n}f(\frac{i}{n})B\textsupsub{$n$}{$i$}(t)$ and the CPETS is $O(\frac{1}{n^{2}})$.

\begin{lemma}
$\forall n > 1$, if the upper bound function $f(t)$ defined in $\left[0,1\right]$ is concave and twice differentiable, for series $s_{i}=f(\frac{i}{n})$, $\forall j \in \left[0, n-1\right]$ and 
$\forall t \in \left[\frac{j}{n}, \frac{j+1}{n}\right]$, note $C(t)=\sum_{i=0}^{n}s_{i}B\textsupsub{$n$}{$i$}(t)$, we have $C(t)$ $-$ $(s_{j}+n(s_{j+1}-s{j})(t-\frac{j}{n}))$=$O(\frac{1}{n^{2}})$, i.e. the difference between the Bezier curve and CPETS is $O(\frac{1}{n^{2}})$.
\label{lemma:mainLemmaDiff}
\end{lemma}

\begin{proof}
We will prove it by recurrence, similar to the proof of lemma \ref{lemma:mainLemma}. For $n \le 2$, it is easy to verify that lemma \ref{lemma:mainLemmaDiff} holds using lemma \ref{lemma:simpleTangentO2Diff}. 

Suppose lemma \ref{lemma:mainLemmaDiff} holds for $n=k$ $(k > 2)$. When $n=k+1$, we only need to show that lemma \ref{lemma:mainLemmaDiff} holds for intervals $\left[\frac{1}{k+1}, \frac{2}{k+1}\right]$, $\left[\frac{2}{k+1}, \frac{3}{k+1}\right]$, ... , 
$\left[\frac{k-1}{k+1}, \frac{k}{k+1}\right]$, because for the first and the last interval, lemma \ref{lemma:mainLemmaDiff} holds due to corollary \ref{corollary:tangentFistLastPointDiff}.

With the assumption of recurrence, $\forall j \in \left[1, k-1\right]$ and 
$\forall t \in \left[\frac{j}{k}, \frac{j+1}{k}\right]$, using the recurrence definition of Bezier curve (\ref{eq:BezierRecurDefinition}), we have:
{
\setlength\abovedisplayskip{1pt}
\setlength\belowdisplayskip{1pt}
\begin{equation}\begin{split}
C(t) &= D(t) + O(\frac{1}{{k}^{2}}) \\
 &= D(t) + O(\frac{1}{(k+1)^{2}})
\label{eq:BezierRecurIniqualityDiff}
\end{split}\end{equation}
}
where $D(t)$ is defined in (\ref{eq:BezierRecurIniqualityRightSide}). One can use lemma \ref{lemma:simpleTangentO2Diff} to verify that 
{
\setlength\abovedisplayskip{1pt}
\setlength\belowdisplayskip{1pt}
\begin{equation}\begin{split}
D(t) = s_{j}+(k+1)(s_{j+1}-s_{j})(t - \frac{j}{k}) + O(\frac{1}{(k+1)^{2}})
\label{eq:shiftedInequalityLinearDiff}
\end{split}\end{equation}
}
$\forall t \in \left[\frac{j}{k+1}, \frac{j+1}{k+1}\right]$, note $\Delta t = \frac{j}{k} - \frac{j}{k+1} > 0$, one can verify that $t+\Delta t \in \left[\frac{j}{k}, \frac{j+1}{k}\right]$, and $\Delta t = O(\frac{1}{(k+1)^{2}})$. Thus $C(t)$ = $C(t+\Delta t) - C^{\prime}(t+\Delta t)\Delta t + o(\Delta t)$ = $C(t+\Delta t) + O(\frac{1}{(k+1)^{2}})$. From (\ref{eq:BezierRecurIniqualityDiff}), (\ref{eq:BezierRecurIniqualityRightSide}) and (\ref{eq:shiftedInequalityLinearDiff}), we have
{
\setlength\abovedisplayskip{1pt}
\setlength\belowdisplayskip{1pt}
\begin{equation}\begin{split}
C(t) &= C(t+\Delta t) + O(\frac{1}{(k+1)^{2}}) \\
&= D(t+\Delta t) + O(\frac{1}{(k+1)^{2}}) \\
&= s_{j}+(k+1)(s_{j+1}-s_{j})(t + \Delta t - \frac{j}{k}) \\
&+ O(\frac{1}{(k+1)^{2}}) \\
&= s_{j}+(k+1)(s_{j+1}-s_{j})(t - (\frac{j}{k} - \frac{j}{k} + \frac{j}{k+1})) \\
&+ O(\frac{1}{(k+1)^{2}}) \\
&= s_{j}+(k+1)(s_{j+1}-s_{j})(t - \frac{j}{k+1}) + O(\frac{1}{(k+1)^{2}})
\label{eq:shiftedInequalityDiff}
\end{split}\end{equation}
}
This concludes the proof for $n=k+1$. Lemma \ref{lemma:mainLemmaDiff} is thus proven.
\end{proof}

Lemma \ref{lemma:mainLemmaDiff} and corollary \ref{corollary:diffUpperBoundCPETS} show that the difference between the upper bound function and the highest possible Bezier curve is $O(\frac{1}{n^{2}})$, where $n$ is the number of control points.

%% file: sections/simulation_results.tex
\section{Simulation results} \label{Simulation}

\begin{figure}[tbp]
\begin{center}
\includegraphics[width=9cm]{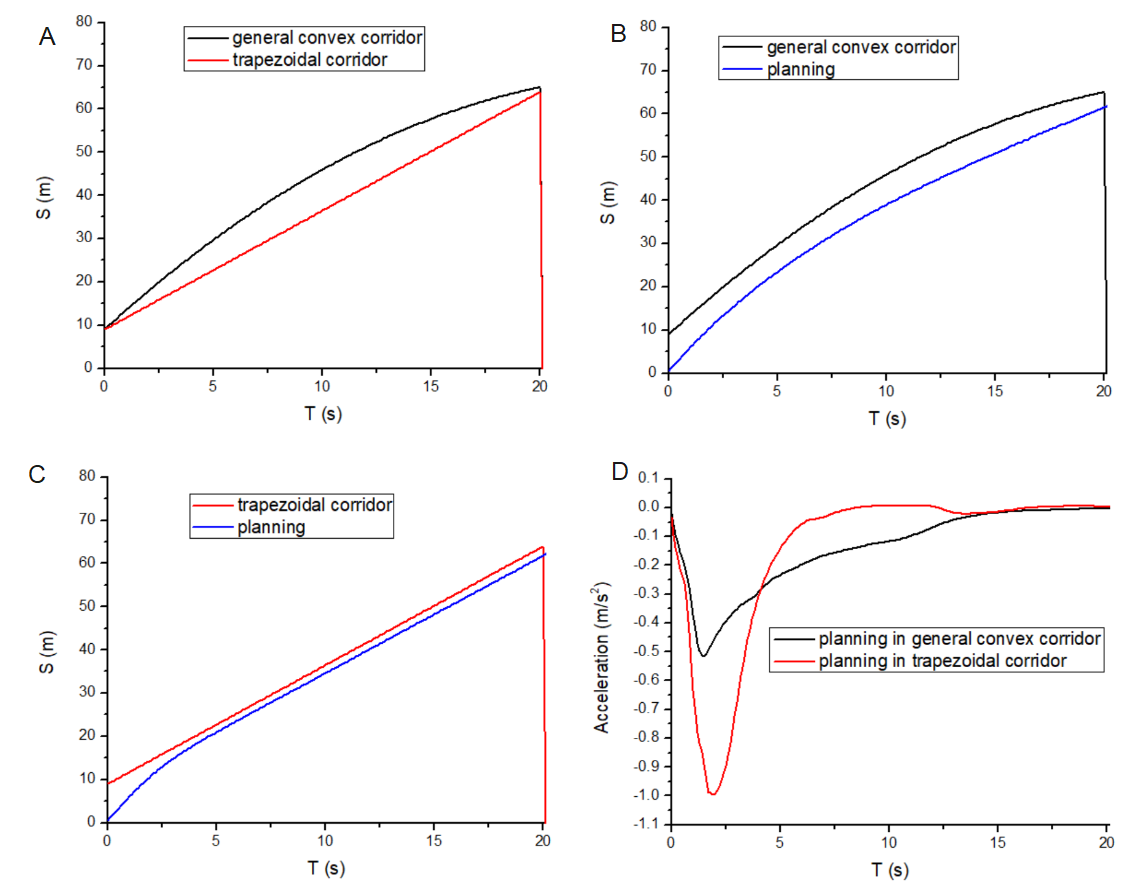}
\end{center}
\vspace{-0.2 in}
\caption{Comparison of performance. A) The shapes of two types of corridors. B) Planning result in general convex corridal. C) Planning result in trapezoidal corridor. D) Comparison of acceleration from two planning results.}
\label{Fig:comparisonOfPerformance}
\vspace{-0.2 in}
\end{figure}

In this section, the motion planning results based on the general convex corridors are compared with trapezoidal corridors-based planning. The same set of parameters (weights, time horizon) were used for both cases to make the results comparable. The scenario includes an ego vehicle and a decelerating front vehicle, similar to the scenario depicted in Fig. \ref{Fig:comparison_of_shapes} A1. The corresponding corridors are shown in Fig. \ref{Fig:comparisonOfPerformance} A. A desirable planning should consider the smoothness and the closeness between the planned trajectory and the reference trajectory. It is also reasonable to consider a higher weight for the trajectory points in the near future compared to the later-time points. Therefore, the cost function is defined as (\ref{eq:costFunction}):
{
\setlength\abovedisplayskip{1pt}
\setlength\belowdisplayskip{1pt}
\begin{equation}\begin{split}
&Cost=\\
&w_{00} \int_{0}^{T_{s}} (S(t)-S_{r}(t))^2 \,dx+w_{01} \int_{T_{s}}^{T_{l}} (S(t)-S_{r}(t))^2 \,dx \\
&+w_{10} \int_{0}^{T_{s}} (\ddot S(t)- \ddot S_{r}(t))^2 \,dx+w_{11} \int_{T_{s}}^{T_{l}} (\ddot S(t)- \ddot S_{r}(t))^2 \,dx \\ 
&+w_{20} \int_{0}^{T_{s}} (\dddot S(t))^2 \,dx+w_{21} \int_{T_{s}}^{T_{l}} (\dddot S(t))^2 \,dx, \\
\label{eq:costFunction}
\end{split}\end{equation}
}
where $w_{00}=5.0$, $w_{01}=2.5$, $w_{10}=10.0$, $w_{11}=3.0$, $w_{20}=25.0$, $w_{21}=10.0$, $T_{l}=20.0$, $T_{s}=6.0$. The terms $T_{l}$ and $T_{s}$ mean that the planning has an extra emphasis on the near future points. The initial speed of the planning is 5.5 $m/s$. $S_{r}(t)$ is a reference trajectory, which is generated by a uniformly accelerated motion with an acceleration of -0.05 $m/s^{2}$. Note that we do not require the reference trajectory to be safe here. A reference trajectory with high acceleration could be unsafe, but it may result in a planned trajectory with a higher acceleration. 

The motion planning problem is then converted into a QP (quadratic programming) problem, whose details are presented in~\cite{ding2019safe} \cite{li2021speed}. OSQP 0.5.1 is adopted as the solver to this problem. As shown in Fig.~\ref{Fig:comparisonOfPerformance}, the general convex corridor results in a larger search space, leading to much smaller deceleration compared to that of the trapezoidal corridor (-0.52 $m/s^{2}$ $vs$ -1.0 $m/s^{2}$) and subsequently less harsh brakes and improved smoothness.

To evaluate the performance of this algorithm in a real-time, multi-frame condition, a cut-in scenario is set up in our simulator (developed based on ROS). The simulation scenario is based on an real-world scenario encountered in road test. The ego vehicle is cruising under a speed limit of 60 $km/h$ (16.77 $m/s$\footnote{Due to the performance of the speed controller, the ego vehicle's actual speed is around 16 $m/s$.}), while another vehicle cuts in at $t=1.0$ s, at a distance of 22 $m$ and at a speed of 7 $m/s$. The cut-in vehicle decelerates at 0.5 $m/s^2$ for 3 seconds, then continues moving at a constant speed. There is an additional 5-meter safety margin behind the cut-in vehicle. The close-loop longitudinal dynamics of the ego vehicle is simulated by a first-order dynamics with a time constant of 0.3 $s$. The corridors are generated with the algorithm presented in~\cite{xin2021enable}~and~\cite{li2021speed}, respectively for general convex shape and trapezoidal shape corridors.

\begin{figure}[tbp]
\begin{center}
\includegraphics[width=9cm]{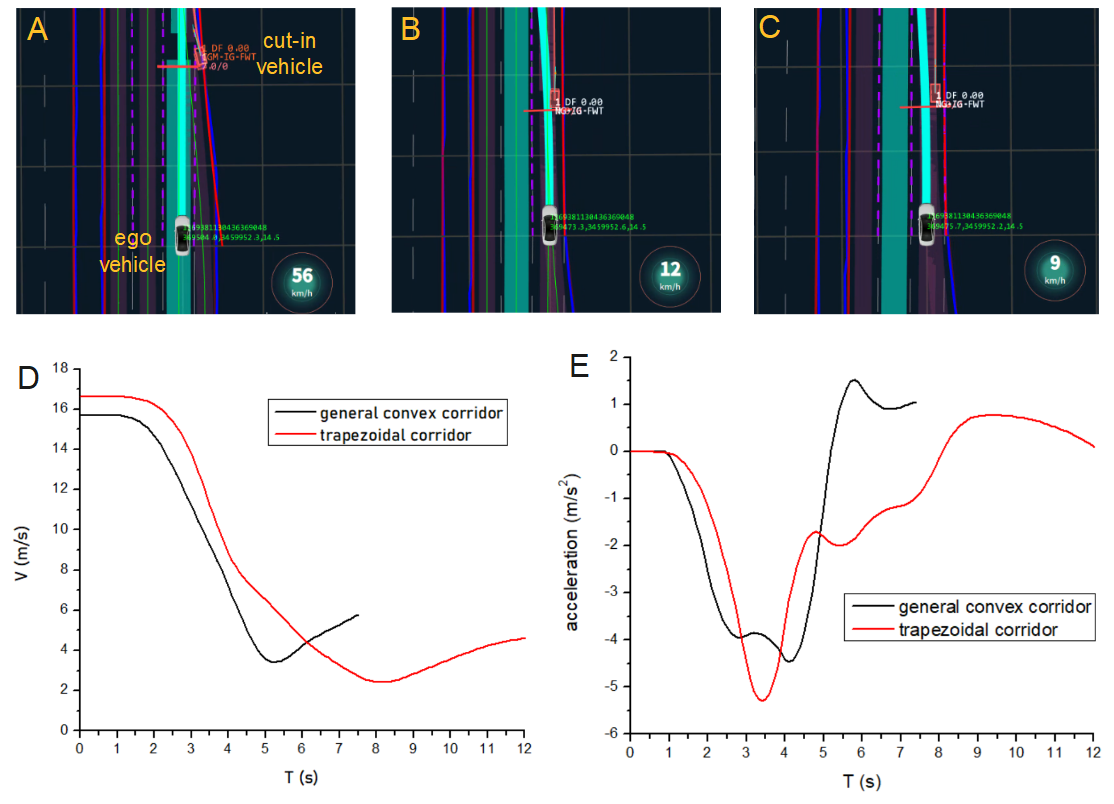}
\end{center}
\vspace{-0.2 in}
\caption{Comparison of performance in simulator. A) The cut-in scenario from top view. B) The lowest speed achieved with general convex corridor. C) The lowest speed achieved with trapezoidal corridor. D) Comparison of ego vehicle's speed from two planning results. E) Comparison of ego vehicle's acceleration from two planning results.}
\label{Fig:comparisonInSimulator}
\vspace{-0.2 in}
\end{figure}

The comparison is shown in Fig.~\ref{Fig:comparisonInSimulator}. The lowest deceleration is -4.46 $m/s^2$ using general convex corridors, compared to -5.23 $m/s^2$ of using trapezoidal corridors. In the case of using general convex corridors, the lowest speed is registered at 3.41 $m/s$, compared to that of trapezoidal corridors at 2.42 $m/s$. Our proposed approach decelerates earlier, which helps in reducing the peak value of deceleration and the length of the deceleration period as an indication of an improved overall comfort.

%% file: sections/discussion.tex
\section{Discussion} \label{Discussion}
This paper contains mainly two parts, a) deriving a sufficient condition for the convex hull property in spatio-temporal corridors, and b) estimating unsearched space. One can note that there is no continuity requirement for the concave upper bound function in a), but requires it to be twice differentiable in b). The requirement in b) can be relaxed to $C^{1}$ using Rolle's remainder of Taylor's theorem but the error is also weakened to $o(\frac{1}{n})$. Given the fact that the difference between the CPETS and the upper bound function is already $O(\frac{1}{n^2})$, we can not achieve better result than $O(\frac{1}{n^2})$.

In this research, the unsearched space is quantified as an order of magnitude instead of an analytical expression. An expression can be obtained by calculating the error each time applying an inequality in the proof of lemma \ref{lemma:mainLemmaDiff}.

This paper is based on Bezier curves, however, since B-spline curves are piece-wise Bezier curve, the theorems in this paper could potentially be extended to B-spline curves with some modifications.

%% file: sections/conclusions.tex
\section{CONCLUSIONS} \label{CONCLUSIONS}
In this paper, we proved a sufficient condition to guarantee the convex hull property for general convex corridors in spatio-temporal frames. This theorem allows for using more complicated shapes to represent constraints as a form of spatio-temporal corridors. It was also shown that the uncovered search space can be shrunk to $O(\frac{1}{n^2})$ compared to $O(1)$ of trapezoidal corridors, which is the best possible result under continuity assumptions.
The use of general convex corridors yields less harsh brakes and enhanced the overall smoothness, specifically in scenarios involving objects with frequent acceleration/deceleration speed profile, which require more complex corridor shapes.

%% file: ICRA22_weize.bbl
\begin{thebibliography}{10}
\providecommand{\url}[1]{#1}
\csname url@samestyle\endcsname
\providecommand{\newblock}{\relax}
\providecommand{\bibinfo}[2]{#2}
\providecommand{\BIBentrySTDinterwordspacing}{\spaceskip=0pt\relax}
\providecommand{\BIBentryALTinterwordstretchfactor}{4}
\providecommand{\BIBentryALTinterwordspacing}{\spaceskip=\fontdimen2\font plus
\BIBentryALTinterwordstretchfactor\fontdimen3\font minus
  \fontdimen4\font\relax}
\providecommand{\BIBforeignlanguage}[2]{{%
\expandafter\ifx\csname l@#1\endcsname\relax
\typeout{** WARNING: IEEEtran.bst: No hyphenation pattern has been}%
\typeout{** loaded for the language `#1'. Using the pattern for}%
\typeout{** the default language instead.}%
\else
\language=\csname l@#1\endcsname
\fi
#2}}
\providecommand{\BIBdecl}{\relax}
\BIBdecl

\bibitem{fan2018baidu}
H.~Fan, F.~Zhu, C.~Liu, L.~Zhang, L.~Zhuang, D.~Li, W.~Zhu, J.~Hu, H.~Li, and
  Q.~Kong, ``Baidu apollo em motion planner,'' \emph{arXiv preprint
  arXiv:1807.08048}, 2018.

\bibitem{sun2018fast}
L.~Sun, C.~Peng, W.~Zhan, and M.~Tomizuka, ``A fast integrated planning and
  control framework for autonomous driving via imitation learning,'' in
  \emph{Dynamic Systems and Control Conference}, vol. 51913.\hskip 1em plus
  0.5em minus 0.4em\relax American Society of Mechanical Engineers, 2018, p.
  V003T37A012.

\bibitem{fleury1995primitives}
S.~Fleury, P.~Soueres, J.-P. Laumond, and R.~Chatila, ``Primitives for
  smoothing mobile robot trajectories,'' \emph{IEEE transactions on robotics
  and automation}, vol.~11, no.~3, pp. 441--448, 1995.

\bibitem{lu2020adaptive}
B.~Lu, G.~Li, H.~Yu, H.~Wang, J.~Guo, D.~Cao, and H.~He, ``Adaptive potential
  field-based path planning for complex autonomous driving scenarios,''
  \emph{IEEE Access}, vol.~8, pp. 225\,294--225\,305, 2020.

\bibitem{xu2014motion}
W.~Xu, J.~Pan, J.~Wei, and J.~M. Dolan, ``Motion planning under uncertainty for
  on-road autonomous driving,'' in \emph{2014 IEEE International Conference on
  Robotics and Automation (ICRA)}.\hskip 1em plus 0.5em minus 0.4em\relax IEEE,
  2014, pp. 2507--2512.

\bibitem{qian2016optimal}
X.~Qian, F.~Altch{\'e}, P.~Bender, C.~Stiller, and A.~de~La~Fortelle, ``Optimal
  trajectory planning for autonomous driving integrating logical constraints:
  An miqp perspective,'' in \emph{2016 IEEE 19th international conference on
  intelligent transportation systems (ITSC)}.\hskip 1em plus 0.5em minus
  0.4em\relax IEEE, 2016, pp. 205--210.

\bibitem{aziz1990bezier}
N.~M. Aziz, R.~Bata, and S.~Bhat, ``Bezier surface/surface intersection,''
  \emph{IEEE computer graphics and applications}, vol.~10, no.~1, pp. 50--58,
  1990.

\bibitem{gao2018online}
F.~Gao, W.~Wu, Y.~Lin, and S.~Shen, ``Online safe trajectory generation for
  quadrotors using fast marching method and bernstein basis polynomial,'' in
  \emph{2018 IEEE International Conference on Robotics and Automation
  (ICRA)}.\hskip 1em plus 0.5em minus 0.4em\relax IEEE, 2018, pp. 344--351.

\bibitem{liu2018convex}
C.~Liu, C.-Y. Lin, and M.~Tomizuka, ``The convex feasible set algorithm for
  real time optimization in motion planning,'' \emph{SIAM Journal on Control
  and optimization}, vol.~56, no.~4, pp. 2712--2733, 2018.

\bibitem{ding2019safe}
W.~Ding, L.~Zhang, J.~Chen, and S.~Shen, ``Safe trajectory generation for
  complex urban environments using spatio-temporal semantic corridor,''
  \emph{IEEE Robotics and Automation Letters}, vol.~4, no.~3, pp. 2997--3004,
  2019.

\bibitem{moghadam2020autonomous}
M.~Moghadam and G.~H. Elkaim, ``An autonomous driving framework for long-term
  decision-making and short-term trajectory planning on frenet space,''
  \emph{arXiv preprint arXiv:2011.13099}, 2020.

\bibitem{li2021speed}
J.~Li, X.~Xie, H.~Ma, X.~Liu, and J.~He, ``Speed planning using bezier
  polynomials with trapezoidal corridors,'' \emph{arXiv preprint
  arXiv:2104.11655}, 2021.

\bibitem{xin2021enable}
L.~Xin, Y.~Kong, S.~E. Li, J.~Chen, Y.~Guan, M.~Tomizuka, and B.~Cheng,
  ``Enable faster and smoother spatio-temporal trajectory planning for
  autonomous vehicles in constrained dynamic environment,'' \emph{Proceedings
  of the Institution of Mechanical Engineers, Part D: Journal of Automobile
  Engineering}, vol. 235, no.~4, pp. 1101--1112, 2021.

\end{thebibliography}
